\documentclass[journal]{IEEEtran}

\usepackage{cite}
\usepackage[pdftex]{graphicx}
\ifCLASSOPTIONcompsoc
   \usepackage[caption=false, font=normalsize, labelfont=sf, textfont=sf]{subfig}
\else
   \usepackage[caption=false, font=footnotesize]{subfig}
\fi
\usepackage{amsmath}
\usepackage{amsthm}
\usepackage{algpseudocode}
\usepackage{algorithm}

\DeclareMathOperator\T{\mathsf{T}}
\DeclareMathOperator\E{\mathsf{E}}
\DeclareMathOperator\tr{\mathsf{tr}}

\DeclareMathOperator\Diag{\mathsf{Diag}}

\newtheorem{lemma}{Lemma}
\newtheorem{definition}{Definition}

\newtheorem{proposition}{Proposition}

\begin{document}

\title{Resilient and Consistent Multirobot Cooperative Localization with Covariance Intersection}

\author{Tsang-Kai~Chang,
        Kenny~Chen,
        and~Ankur~Mehta%
\thanks{The authors are with the Department
of Electrical and Computer Engineering, University of California, Los Angeles,
CA, 90095 USA (e-mail: tsangkaichang@ucla.edu; kennyjchen@ucla.edu; mehtank@ucla.edu).}}

\maketitle

\begin{abstract}
   Cooperative localization is fundamental to autonomous multirobot systems, but most algorithms couple inter-robot communication with observation, making these algorithms susceptible to failures in both communication and observation steps. To enhance the resilience of multirobot cooperative localization algorithms in a distributed system, we use covariance intersection to formalize a localization algorithm with an explicit communication update and ensure estimation consistency at the same time. We investigate the covariance boundedness criterion of our algorithm with respect to communication and observation graphs, demonstrating provable localization performance under even sparse communications topologies. We substantiate the resilience of our algorithm as well as the boundedness analysis through experiments on simulated and benchmark physical data against varying communications connectivity and failure metrics.  Especially when inter-robot communication is entirely blocked or partially unavailable, we demonstrate that our method is less affected and maintains desired performance compared to existing cooperative localization algorithms.
\end{abstract}


\IEEEpeerreviewmaketitle

\section{Introduction}

   % importance of multirobot system/localization
   Localization is one of the most fundamental elements for autonomous mobile robots. As multiple robots form a team to improve robustness and scalability, localization of a multirobot system is therefore a premise to the successful deployment of such system to achieve high-level goals.

   Different from a single-robot scenario, there are two additional aspects that enable multiple robots to localize themselves cooperatively. First, a robot can observe other robots and the relative observation between them can enhance the overall localization performance. Second, robots can share their information with one another, which can also improve the overall localization performance. This scheme is called \emph{cooperative localization}. Different approaches are proposed for multirobot cooperative localization, ranging from the extended Kalman filter (EKF), the particle filter, to optimization-based approaches. Among all approaches of cooperative localization, we focus on the EKF-based approaches primarily due to its computational efficiency.
   
   % consistency: challenges of this topic
   While cooperative localization takes advantage of relative observation and inter-robot communication, designing a cooperative localization algorithm has its own difficulties. By considering localization as an estimation problem, estimation consistency remains a challenge for multirobot systems. Intuitively, as a correlation implies the dependency between two estimates, if these two estimates are fused with underestimated correlation, the resulting estimate no longer accounts for the estimation uncertainty and the \emph{over-confidence problem} occurs. An extreme example is to regard two repetitive datapoints as two independent information in data fusion, and therefore some researchers also refer this problem as the \emph{double counting problem}. 
   In EKF, the linearization step also results in the inconsistency problem, mainly due to a mismatch from the linear state space model and the Gaussian assumption \cite{julier_counter_2001, bailey_consistency_2006, huang_convergence_2007}. We stay with the assumption that linearization is reasonable and address the inconsistency problem associated with inter-robot correlations.

   % communication
   In order to keep each inter-robot correlation updated, communication is often extensively performed, but excessive communication poses resilience concern. 
   In particular, the resilience of the cooperative localization algorithm ensures the performance does not drastically fall off in face of the communication failures or even adversaries.
   In the seminal cooperative localization algorithm \cite{roumeliotis_distributed_2002}, an all-to-all communication is required after every observation in order to maintain correlations equivalent to a centralized EKF. Tracking correlations in a distributed system not only requires extensive communications, but it also makes the system vulnerable to even a single communication failure. The following works attempt to decrease the amount of communication either by introducing additional server in the distributed system \cite{kia_server_assisted_2018}, or by sacrificing the estimation consistency \cite{li_split_2013, luft_recursive_2018}. 
   For either improvement, communication is regarded as one of the steps in an observation update and takes place right after a relative observation. However, such association strongly relies on the assumption that communication is available whenever needed. In short, in most cooperative localization algorithms, communication between robots is either excessive or is implicitly assumed to be always available and free from failure, which makes these algorithms less resilient.

   %%%
   % our approach: explicit communication, distributed estimation algorithm
   To maintain estimation consistency and to ensure communication resilience, we separate the communication step from the observation step in the proposed algorithm by using the covariance intersection (CI) fusion technique \cite{julier_nondivergent_1997, chen_estimation_2002, sun_multisensor_2004, reinhardt_minimum_2015}. In this algorithm, the estimation consistency is directly assured by CI. Since the communication step is explicit in the proposed algorithm, communication is no longer a complementary part after the observation but rather an independent source contributing further information. Therefore, the proposed algorithm does not need excessive communications, and communication unavailability will not affect our algorithm's observation update which enhances the resilience.
   CI-based cooperative localization algorithms are commonly criticized by having an overly conservative estimate. This work fundamentally avoids this overly conservative estimation problem by limiting CI only to the communication update, and it is shown that the localization performance is comparable to the algorithm based on the centralized EKF.

   % analysis
   In addition to the proposal of the algorithm, we complete this paper with a performance analysis on the proposed algorithm. By interpreting the proposed algorithm as a distributed estimation algorithm, we investigate the covariance boundedness criterion to assure an upper bound on localization performance. To address the nature of multiagent system, the analysis takes the configuration of observation and communication into account with graph description.

% ==============================   

   % different from conf publication
   This paper is a revised and substantially extended version of our previous conference publication \cite{tsangkai_1}. The conference paper aims to formalize the algorithm, but the investigation of the resilience is only presented in this paper. Moreover, we extensively consider various scenarios in this paper to study the effect of communication failures on different algorithms. Technically, the conference paper assumes that the orientation estimate is given, and we drop that assumption to ensure the applicability of the proposed algorithm in this paper.

   % layout
   % We organize this paper as follows:

\section{Related Work}   
\label{sec:related_work}

   The concept of cooperative localization is first proposed in \cite{kurazume_cooperative_1994}, and the term ``cooperative localization'' is later coined in \cite{rekleitis_multiagent_1998}. The work \cite{kurazume_experimental_2000} extends the techniques in \cite{kurazume_cooperative_1994} to an experimental setting. The cooperative localization is also developed in a team of small robots \cite{grabowski_localization_2001} to globally localize the team. While these algorithms are able to use the relative observation and the inter-robot communication, they are limited to particular system settings.
    In the early stage of the cooperative localization development, these algorithms establish the basis for current cooperative localization algorithms that are more fundamental and general.

   Depending on the underlying framework, we classify current cooperative localization algorithms into three categories: EKF-based approaches, particle filter-based approaches, and optimization-based approaches. In this section, we highlight the advantages of each algorithm as well their limitations, with specific focus on estimation consistency and communication resilience.

\subsection{Extended Kalman Filter-based Approaches}   
%% Centralized EKF one 
% exact
   
   EKF-based approaches are the mainstream for cooperative localization algorithms. The benchmark of the EKF-based approach is established in the seminal paper \cite{roumeliotis_distributed_2002}, together with its theoretical analyses in \cite{mourikis_performance_2006, huang_observability_based_2011}. This cooperative localization algorithm emphasizes the importance of correlations between inter-robot estimates, and is fundamentally free from the over-confidence problem. In particular, this algorithm is exactly a distributed implementation of the centralized Kalman filter, and the correlations between inter-robot estimates are well tracked and updated. However, the communication cost for this distributed implementation is very high. In particular, an all-to-all communication is needed after every observation. As a result, the algorithm performance is very susceptible to the communication failure. The estimation consistency no longer holds with a single communication failure, which impairs the system's overall resilience.
   To lower the communication cost in \cite{roumeliotis_distributed_2002}, a server in charge of calculating and broadcasting the estimation information is introduced in \cite{kia_server_assisted_2018}. Therefore, all-to-all communication is no longer necessary to recover exact inter-robot correlations. However, the introduction of the server makes the whole system less distributed. Not only is the entire system more vulnerable to the server's failure, but also an initial setup of the server is required.

   Beyond the proposal of the algorithm itself, the theoretical analysis of this centralized-equivalent algorithm is reported in \cite{mourikis_performance_2006, huang_observability_based_2011}. In \cite{mourikis_performance_2006}, with an implicit assumption that all-to-all communication is available and successful at all times, the observation configuration criterion of the bounded covariance is thoroughly investigated. We conduct a similar boundedness analysis for our algorithm, especially on the observation and the communication configurations without the assumption of perfect communication. In \cite{huang_observability_based_2011}, the linearization consistency issue of EKF-based localization algorithms is presented, with focus on the linearization points while calculating the Jacobians. In this paper, we maintain the assumption that linearization error is small, and focus on the linear estimation in the performance analysis in Sec. \ref{sec:analysis}.

   % misc
   % can put anywhere
   % There are some other works of cooperative localization with exact inter-robot correlations on specific scenarios. A cooperative localization algorithm is proposed in \cite{bailey_decentralised_2011} with a server in charge of information fusion. For the underwater vehicles with acoustic modems, a cooperative localization algorithm is studied in \cite{webster_decentralized_2013}. These two algorithms rely on the presence of a server to recover the exact inter-robot correlations, and are thus not suitable solution for distributed robots. In addition, these two algorithms are presented and tested in systems with less than 3 robots, and it is unclear how to extend the algorithm to systems beyond that.     

%% both avoid fusion of correlated estimates

   The over-confidence problem can be avoided by only fusing uncorrelated estimates in a multirobot system. This idea is substantiated by keeping a bank of filters in each robot \cite{bahr_consistent_2009, zhu_multirobot_2019}. In \cite{bahr_consistent_2009}, each underwater vehicle maintains a bank of EKFs and only uncorrelated estimates in the bank are fused subsequently. A similar localization method is proposed for mobile robots while simultaneously tracking targets in \cite{zhu_multirobot_2019}. The main disadvantage of this category is that the number of the filters in a single bank grows exponentially with the number of the robots in a system, which imposes significant storage cost. Another way to realize this idea, called the state exchange scheme, is proposed in \cite{karam_localization_2006}. Specifically, there is no fusion but rather replacement within robot estimates to maintain the independence between robot estimates. Historical information is therefore discarded with the arrival of new information, which leads to extremely inefficient estimates.

%% approximate the correlation   
   % CI
   Instead of retrieving exact inter-robot correlations, some approaches approximate the correlations and thus largely decrease communication cost. Covariance intersection is often applied in these approaches, since it can fuse several estimates without knowing the correlations and maintain estimation consistency at the same time. To the best of our knowledge, the first application of CI for cooperative localization is the example in \cite{arambel_covariance_2001}. Our algorithm is similar to this one, but we generalize the algorithm and systematically study the boundedness criterion in this paper. CI is also applied differently in the cooperative localization in \cite{carrillo_arce_decentralized_2013}. In this algorithm, each robot only keeps its own state estimate, and the relative observation is fused by CI. As a consequence, the estimation is very conservative in this method. On the contrary, in our algorithm, robots keep an estimate of the entire system, and relative observations can directly update the estimate. Consequently, our algorithm is not overly conservative, as verified by the experiments in Sec. \ref{sec:experiment}. An extended work of \cite{carrillo_arce_decentralized_2013} is presented in \cite{klingner_fault_tolerant_2019} by incorporating the covariance union. However, as an even more conservative fusion scheme than CI, this algorithm with covariance union is too conservative for any practical use.

   % approximated but not consistent
   While algorithms based on CI ensure estimation consistency, some other algorithms with approximated correlations do not maintain such property. The split covariance intersection is applied in cooperative localization in \cite{li_split_2013}. The main drawback of this approach is that the independent part and the dependent part can not be clearly split. Therefore, the fusion in relative observation is problematic, as mentioned in \cite{carrillo_arce_decentralized_2013}. In \cite{luft_recursive_2018}, the exact covariance matrix is approximated by a block diagonal matrix, and the inter-robot correlations are thus suppressed. Since the estimation consistency is not maintained, the over confidence problem can occur when applying this algorithm. 
   A cooperative localization algorithm that targets at the scenario with measurements at different time instances is proposed in \cite{indelman_graph_based_2012}. However, the fundamental Kalman filtering assumption of the noise independence has to be contradicted to avoid recursive updates among robots, which also raises the same concerns of estimation consistency.

\subsection{Particle Filter-based Approaches}

   % particle filter   
   To alleviate the nonlinearity issues in multirobot localization, particle filters are often applied \cite{fox_probabilistic_2000, howard_putting_2003, prorok_low_cost_2012}. However, the correlations among robot estimation is not easy to handle in particle filter-based cooperative localization algorithms. One of the early attempts to applying particle filter in cooperative localization can be found in \cite{fox_probabilistic_2000}. However, the correlations between robots are ignored, and the result is overly confident.
   In \cite{howard_putting_2003}, a dependency tree is introduced to alleviate the double counting problem between two robots, but it only avoids the most obvious cases and still cannot prevent the over-confidence problem from happening. In \cite{prorok_low_cost_2012}, a particle clustering method is introduced to reduce the computational complexity of particle filter-based methods, but correlations between estimates are not explicitly addressed. In fact, the authors wrongly assume the independence between the estimates in different robots in reciprocal sampling.
   
   In summary, particle filter-based multirobot cooperative localization algorithms cannot track the correlations between distributed estimates easily. Moreover, particle filters are generally more computationally expensive compared to Kalman filter-based approaches.

\subsection{Optimization-based Approaches}
   
   Cooperative localization can also be solved with optimization-based methods, including maximum likelihood estimation \cite{howard_localization_2003} and maximum a posteriori estimation \cite{nerurkar_distributed_2009}. Optimization-based approaches first formulate cooperative localization as a nonlinear least squares problem in a centralized fashion, and then is directly solved offline. To counter the centralized modeling and offline solving for localization, excessive communication is necessary between distributed robots. As a result, in both algorithms \cite{howard_localization_2003, nerurkar_distributed_2009}, robots have to broadcast their information to the entire team regularly. In terms of the offline nature, the authors in \cite{nerurkar_distributed_2009} partially tackle this problem by marginalization, but an all-to-all communication is expected in this marginalization step.
   For optimization-based approaches, the burden of communication makes them less popular compared to those aforementioned algorithms.

\section{Mathematical Preliminaries}

   In this section, we provide the essential mathematical preliminaries to construct and analyze the proposed cooperative localization algorithm.

%%%
\subsection{Estimation Consistency and Covariance Intersection}   

   A consistent estimate can be seen as a conservative estimate regarding the estimation uncertainty intuitively. In other words, a conservative estimate reports larger uncertainty than the estimate really provides, so as to avoid over-confidence data fusion.
   As we use Gaussian random vectors as estimates, since the covariance matrix represents the uncertainty of the estimate, a consistent estimate can be considered as an estimate that has larger covariance matrix in the positive definite sense.
    The aforementioned over-confidence problem and the double counting problem can be avoided if estimation consistency is maintained. Formally, a consistent estimate is mean-preserving and has no smaller covariance matrix in the positive definite sense, with the following definition:
   \begin{definition}[Estimation consistency]
    An estimate $\hat{z}$ of a real vector $z$ is a Gaussian random vector with mean $\E[\hat{z}]$ and covariance $\Sigma_{\hat{z}}$. The estimation $\hat{z}'$ of $z$ is called \emph{consistent} of $\hat{z}$ if $\E[\hat{z}']=\E[\hat{z}]$ and $\Sigma_{\hat{z}'} \geq \Sigma_{\hat{z}}$.
   \end{definition}  
   
   %%% CI
   \begin{lemma}[Covariance intersection \cite{julier_nondivergent_1997, chen_estimation_2002, sun_multisensor_2004, reinhardt_minimum_2015}]
      Given $N$ consistent estimates $\hat{z}_i$ of $\hat{z}$ with covariances $\Sigma_{i}$ for $i=1,\dots,N$, the estimate $\hat{z}'$ is also consistent of $\hat{z}$ with
      \begin{equation}
         \Sigma_{\hat{z}'}^{-1} = \sum_{i=1}^N c_i \Sigma_{i}^{-1},
         \label{eq:ci_mean}
      \end{equation}
      and 
      \begin{equation}
         \E[\hat{z}'] = \Sigma_{\hat{z}'} \sum_{i=1}^N c_i \Sigma_{i}^{-1} \E[\hat{z}_i],
         \label{eq:ci_cov}
      \end{equation}
      where the nonnegative coefficients $c_i$ satisfy $\sum_{i=1}^N c_i = 1$.
      \label{lemma:ci}
   \end{lemma}
   
   CI is able to combine several consistent estimates which might be correlated, and the result stays consistent. 
   The nonnegative coefficients $\{c_i, i=1,\dots, N\}$ such that $\sum_{i=1}^N c_i = 1$ are called the \emph{convex coefficients}.
   
   As all the estimates are Gaussian random vectors, they can be represented in the information form, which leads to a compact formula of CI. By defining the information mean $\bar{e}_i = \Sigma_{i}^{-1} \E[\hat{z}_i]$ and the information matrix $I_i = \Sigma_{i}^{-1}$, equations (\ref{eq:ci_mean}) and (\ref{eq:ci_cov}) can be rewritten as
   \begin{equation}
      \Sigma_{\hat{z}'}^{-1} = \sum_{i=1}^N c_i I_{i},
   \end{equation}
   and
   \begin{equation}
      \Sigma_{\hat{z}'}^{-1} \E[\hat{z}'] = \sum_{i=1}^N c_i \bar{e}_i,
   \end{equation}
   and $(\Sigma_{\hat{z}'}^{-1} \E[\hat{z}'], \Sigma_{\hat{z}'}^{-1})$ is the information form of $(\E[\hat{z}'], \Sigma_{\hat{z}'})$. In short, CI is actually the convex combination of the information means and that of the information matrices. We will use the information form since the infinite uncertainty is easier to characterize in this form.

%%%
\subsection{Graph Theory}

   A directed graph $\mathcal{G}=(V, E_\mathcal{G})$ is applied to characterize the configuration of communication and observation of the multirobot system. In the graph $\mathcal{G}$, the vertex set $V$ contains all the agents, i.e. $V=\{1,\dots,N\}$, and an edge is an ordered pair $(j,i) \in E_\mathcal{G}$, $j \neq i$. We may refer an edge as a \emph{link} in this paper. A \emph{path} in $\mathcal{G}$ is given by a sequence of vertices $ (v_{i_1}, v_{i_2}, \dots, v_{i_{m+1}})$ such that $(v_{i_k}, v_{i_{k+1}}) \in E_\mathcal{G}$ for $k=1,\dots, m$. 
   The graph $\mathcal{G}$ is called \emph{strongly connected} if there is a path for every pair of vertices. The graph $\mathcal{G}$ is called \emph{weakly connected} if there is a path between every pair of vertices regardless of the edge direction.
   The \emph{neighborhood} of agent $i$ is defined as $N_{\mathcal{G}}(i) = \{j | (j,i) \in E_\mathcal{G}\}$. The \emph{inclusive neighborhood} of agent $i$ is defined by $N_{\mathcal{G}}^*(i)= N_{\mathcal{G}}(i) \cup \{i\}$. 
   A complete treatment of graph theory on multiagent systems can be found in \cite{mesbahi_graph_2010}.

\section{Multirobot Cooperative Localization Algorithm}

   % setup
   We consider a multi-robot system in the 2D scenario with $N$ robots. At time $t$, the spatial state of robot $i$ is given by $q_{i,t}=[\theta_{i,t}, p^{\T}_{i,t}]$, which includes the orientation $\theta_{i,t}$ and the Cartesian position $p_{i,t}=[x_{i,t},y_{i,t}]^{\T}$, for $i \in \{1,\dots,N\}$. We assume that spatial states across all robots are in a common reference frame, which can be initialized by the cooperative localization setting \cite{trawny_interrobot_2010, zhou_determining_2013}.
   The robots can observe several distinguishable landmarks whose positions are given. While all landmarks serve as a reference for absolute spatial state, without loss of generality, we consider a single landmark in the environment, denoted by $\lambda$.

   % estimation
   In the proposed cooperative localization algorithm, robot $i$ has to track its own spatial state and the positions of other robots. We can represent the state of the entire system estimated by robot $i$ as 
   \begin{equation}
      s_{i,t} = 
      \begin{bmatrix}
         p_{1,t}^{\T}, \dots, q_{i,t}^{\T}, \dots, p_{N,t}^{\T}
      \end{bmatrix}^{\T}.
      \label{eq:state_def}
   \end{equation}
   We consider that case where the orientations of other robots are not tracked by robot $i$. The state defined in (\ref{eq:state_def}) is similar to the one in the EKF SLAM \cite{dissanayake_solution_2001}, where those $p_{i,t}$ in (\ref{eq:state_def}) are not stationary but dynamic. The proposed algorithm remains valid when the orientations of other robots are tracked with various sensing modalities. In fact, the proposed algorithm will be easier if all robots track the same state. We instead demonstrate the necessary steps when the robots track different states in the model (\ref{eq:state_def}).
   
   Based on the Kalman filtering, robot $i$ keeps a Gaussian estimate of $s_{i,t}$, denoted by $\hat{s}_{i,t}$, with mean $\bar{s}_{i,t}$ and covariance $\Sigma_{i,t}$. Depending on the type of arriving information, the proposed cooperative localization algorithm contains three updates: 
   \begin{itemize}
      \item the \emph{time propagation update} at the arrival of the proprioceptive information, 
      \item the \emph{observation update} at the arrival of the exteroceptive information, and
      \item the \emph{communication update} at the inter-robot communication.
   \end{itemize}
   The proposed algorithm does not require communication after the inter-robot observation. Therefore, all these three sources of informations contribute independently and complementarily to achieve localization.

%%%
\subsection{Time Propagation Update}

   % model
   The time propagation update is performed when robot $i$ has the proprioceptive information of the system, which consists of its own odometry input and those of other robots. Robot $i$ has the odometry input $u_{i,t}$, and estimates its next spatial state by a generic motion model:
   \begin{equation}
      q_{i,t+1}
       = f (q_{i,t}, u_{i,t}+ w_{i,t}),
      \label{eq:motion_model}
   \end{equation}
   where $w_{i,t}$ is the input noise and it is modeled as a zero-mean Gaussian random vector with covariance matrix $Q_w$.
   
   % other robot
   The odometry inputs for other robots, $u_{j,t}$, $j \neq i$, however, are not available for robot $i$. Without the exact value, we regard $u_{j,t}$ as a random variable, and the variability of that random variable is large enough to incorporate all possible values and to ignore the noise effect. The goal is not to guess the odometry input of other robots, but to maintain large estimation uncertainty that the estimate can be corrected during the observation or the communication updates. To be specific, we model the input $u_{j,t}$ as a Gaussian random vector with covariance matrix $Q_u$ large enough to maintain the estimation consistency. That is, for robot $i$,
   \begin{equation}
      p_{j,t+1} = f_p (p_{j,t}, u_{j,t}), \quad j \neq i.
   \end{equation}
   As the input noise in each robot is independent, the time update for $\hat{s}_{i,t}$ can be easily obtained, as in Algorithm \ref{al:time_update}.

\begin{algorithm}[t]
  \caption{The time propagation update for robot $i$}
  \begin{algorithmic}
    \State \textbf{Input}: $\bar{s}_{i,t}$, $\Sigma_{i,t}$, $u_{i,t}$
    \State \textbf{Output}: $\bar{s}_{i,t+1}$, $\Sigma_{i,t+1}$
    \newline
    \State\quad $\bar{s}_{i,t+1} = \begin{bmatrix}
        f_p(\bar{p}_{1,t}, \E[u_{1,t}]) \\
        \vdots \\
        f(\bar{q}_{i,t}, u_{i,t})  \\
        \vdots \\
        f_p(\bar{p}_{N,t}, \E[u_{N,t}])
        \end{bmatrix}$.
     \State\quad $F_i = \frac{\partial f(q, u)}{\partial q}(\bar{q}_{i,t}, u_{i,t})$.
     \State\quad $F_j = \frac{\partial f_p(p, u)}{\partial p}(\bar{p}_{j,t}, \E[u_{j,t}])$, for $j \neq i$.
     \State\quad $G_i = \frac{\partial f(q, u)}{\partial u}(\bar{q}_{i,t}, u_{i,t})$.
     \State\quad $G_j = \frac{\partial f_p(p, u)}{\partial u}(\bar{p}_{j,t}, \E[u_{j,t}])$, for $j \neq i$.
     \State\quad $Q = \Diag(G_1 Q_u G_1^{\T}, \dots, G_i Q_w G_i^{\T}, \dots, G_N Q_u G_N^{\T})$.
     \State\quad $\Sigma_{t+1} = \Diag(F_1,\dots,F_N) \Sigma_{t} \Diag(F_1,\dots,F_N)^{\T} + Q$.
     \end{algorithmic}
     \label{al:time_update}
\end{algorithm}

\subsection{Observation Update}   
   
   When robot $i$ observes either the landmark or other robots, the observation update is performed with the exteroceptive information.
   Specifically, robot $i$ observes the landmark in the environment according to model
   \begin{equation}
      o_{i\lambda,t} = h_{i\lambda}(q_{i,t}) + v_{i\lambda,t},
   \end{equation}
   where the $v_{i\lambda,t}$ is the observation noise modeled as zero-mean Gaussian with covariance $R_{i\lambda,t}$.
   The relative observation model between two robots is similarly given as
   \begin{equation}
      o_{ij,t} = h_{ij}(q_{i,t}, p_{j,t}) + v_{ij,t}.
   \end{equation}

   % combine all observations
   In reality, robot $i$ can observe more than one object at the same time, and the observation results may therefore be correlated. Thus, we define the set $O_{i,t}$ as the set of objects that robot $i$ observes at time $t$, including both landmarks and robots.
   With $O_{i,t} = \{i_1, i_2, \dots, i_{n_i}\}$, we stack all the measurements at time $t$ into the vector $o_{i,t}$,
   \begin{equation}
      o_{i,t} = \begin{bmatrix}
         o_{i i_1, t} \\
         \vdots \\
         o_{i i_{n_i}, t}
      \end{bmatrix}
      = \left[o_{ij,t} \right]_{j \in O_{i,t}},
   \end{equation}
   together with the entire observation noise $v_{i,t} = \left[v_{ij,t} \right]_{j \in O_{i,t}}$.
   With the covariance of the noise $v_{i,t}$ denoted by $R_{i,t}$, we have the EKF observation updates:
   \begin{equation}
      \bar{s}_{i,t^+} = \bar{s}_{i,t} + \Sigma_{i,t} H_i^{\T} ( H_i \Sigma_{i,t} H_i^{\T} + R_{i,t} )^{-1}(o_{i,t} - H_i \bar{s}_{i,t}),
   \end{equation}
   and
   \begin{equation}
      \Sigma_{i,t^+}^{-1} = \Sigma_{i,t}^{-1} + H_i^{\T} R_{i,t}^{-1} H_i,
   \end{equation}
   where the observation matrix is the stacked matrix given by
   \begin{equation}
      H_i = \left[ \frac{\partial h_{ij}(s)}{\partial s}(\bar{s}_{i,t}) \right]_{j \in O_{i,t}}.
   \end{equation}
   % Since there is no time elapse in the observation update, we use the notation $t^+$ to denote the instance after the observation update at time $t$.

\begin{algorithm}[t]
   \caption{The observation update for robot $i$}
   \begin{algorithmic}
      \State \textbf{Input}: $\bar{s}_{i,t}$, $\Sigma_{i,t}$, $o_{i,t}$
      \State \textbf{Output}: $\bar{s}_{i,t^+}$, $\Sigma_{i,t^+}$
      \newline
      \State\quad $H_i = \left[ \frac{\partial h_{ij}(s)}{\partial s}(\bar{s}_{i,t}) \right]_{j \in O_{i,t}}$.
      \State\quad $\bar{s}_{i,t^+} = \bar{s}_{i,t} + \Sigma_{i,t} H_i^{\T} ( H_i \Sigma_{i,t} H_i^{\T} + R_{i,t} )^{-1}(o_{i,t} - H_i \bar{s}_{i,t})$.
      \State\quad $\Sigma_{i,t^+}^{-1} = \Sigma_{i,t}^{-1} + H_i^{\T} R_{i,t}^{-1} H_i$.
   \end{algorithmic}
\end{algorithm}

\subsection{Communication Update}

   When robot $j$ sends its estimation information, in particular $\bar{s}_{j,t}$ and $\Sigma_{j,t}$, to robot $i$, robot $i$ can use this information to update its own estimation. However, the correlation between $\hat{s}_{i,t}$ and $\hat{s}_{j,t}$ is hard to track in a distributed system. Without knowing the exact correlations, we use CI to fuse these estimates to maintain the estimation consistency. 
   
   The direct application of CI by (\ref{eq:ci_mean}) and (\ref{eq:ci_cov}) is problematic, because $\hat{s}_{i,t}$ and $\hat{s}_{j,t}$ do not estimate the same state. In particular, the orientation estimate $\theta_{i,t}$ is in $\hat{s}_{i,t}$ but not in $\hat{s}_{j,t}$. In order to ensure that $\hat{s}_{i,t}$ and $\hat{s}_{j,t}$ represent the same state, we first have to remove the estimate of $\theta_j$ from $\hat{s}_{j,t}$, and then add the dummy estimate of $\theta_i$. We denote the resulting estimate as $\hat{s}^j_{i,t}$ and then the CI can be applicable at robot $i$.
   
   To remove the estimate of $\theta_j$ from $\hat{s}_{j,t}$, we use a $2N \times (2N+1)$ matrix defined by
   \[
      [T_{j^-}]_{m,n} = 
      \begin{cases}
         1       & \quad \text{if } m=n, n \leq 2(j-1)\\
         1       & \quad \text{or } m=n-1, n \geq 2j\\
         0  & \quad \text{otherwise}
      \end{cases}.
  \]
  Therefore, $T_{j^-}\hat{s}_{j,t}$ will be the estimate of $[p_{1,t}^{\T}, \dots, p_{N,t}^{\T}]$ with mean $T_{j^-} \bar{s}_{j,t}$ and covariance matrix $T_{j^-} \Sigma_{j,t} T_{j^-}^{\T}$. Equivalently, the same estimate admits an information form with the information mean $(T_{j^-} \Sigma_{j,t} T_{j^-}^{\T})^{-1} T_{j^-} \bar{s}_{j,t}$ and the information matrix $(T_{j^-} \Sigma_{j,t} T_{j^-}^{\T})^{-1}$.

  Next, we insert $\theta_i$ to the estimate $T_{j^-}\hat{s}_{j,t}$ in the information form to obtain $\hat{s}^j_{i,t}$. Since there is no information of $\theta_i$ from robot $j$, this step just ensures that the corresponding terms in the vector are matched, and the variance of $\theta_i$ in $\hat{s}^j_{i,t}$ will be infinite. We use a $(2N+1) \times 2N$ matrix,
   \[
      [T_{i^+}]_{m,n} = 
      \begin{cases}
         1       & \quad \text{if } m=n, n \leq 2(i-1)\\
         1       & \quad \text{if } m=n+1, n \geq 2i\\
         0  & \quad \text{otherwise}
      \end{cases},
  \]
  to append $\theta_i$. Thus, the information mean of $\hat{s}^j_{i,t}$ will be $T_{i^+} (T_{j^-} \Sigma_{j,t} T_{j^-}^{\T})^{-1} T_{j^-} \bar{s}_{j,t}$, and the corresponding information matrix will be $T_{i^+} (T_{j^-} \Sigma_{j,t} T_{j^-}^{\T})^{-1} T_{i^+}^{\T}$. By this construction, the exact mean of $\theta_{i,t}$ in the estimate $\hat{s}^j_{i,t}$ is not important, since the corresponding variance is infinite, and will not affect the result of CI.

   We define the set $C_{i,t}$ to contain all robots whose information is received at robot $i$ at time $t$, and $C_{i,t}^* = C_{i,t} \cup \{i\}$. Together with the convex coefficient $\{c_j, j \in C_{i,t}^*\}$, we have the communication update described in Algorithm \ref{al:comm_update}.
   % Determining the coefficient $\{c_j, j \in C_{i,t}^*\}$ is often based on a minimization of $\dett(\Sigma_{i,t^+})$ or that of $\tr(\Sigma_{i,t^+})$, and can be found in related covariance intersection literature (citation). 

\begin{algorithm}[t]
   \caption{The communication update for robot $i$}
   \begin{algorithmic}
      \State \textbf{Input}: $\bar{s}^j_{t}$, $\Sigma_{j,t}$, $j \in C^*_{i,t}$
      \State \textbf{Output}: $\bar{s}^i_{t^+}$, $\Sigma_{i,t^+}$
      \newline
      \State To construct the information form:
      \State\quad $\bar{e}^{j}_{i,t} = T_{i^+} (T_{j^-} \Sigma_{j,t} T_{j^-}^{\T})^{-1} T_{j^-} \bar{s}_{j,t}$, $j \in C_{i,t}$
      \State\quad $I^{j}_{i,t} = T_{i^+} (T_{j^-} \Sigma_{j,t} T_{j^-}^{\T})^{-1}  T_{i^+}^{\T}$, $j \in C_{i,t}$
      \newline
      \State To fuse incoming estimates by CI:
      \State\quad $\bar{s}_{i,t^+} = \Sigma_{i,t^+} \left( \sum_{j \in C_{i,t}} c_j \bar{e}^{j}_{i,t} + c_i \Sigma_{i,t}^{-1} \bar{s}_{i,t} \right)$.
      \State\quad $\Sigma_{i,t^+}^{-1} = \sum_{j \in C_{i,t}} c_j I^{j}_{i,t} + c_i \Sigma_{i,t}^{-1}$.
   \end{algorithmic}
   \label{al:comm_update}
\end{algorithm}

\section{Boundedness Analysis of the Position Estimation Covariance}
\label{sec:analysis}

   % motivation
   For the localization algorithm, the boundedness of the covariance matrix ensures that the estimation uncertainty is limited, which is essential for the success of the high-level tasks. Whether the estimation covariance matrix of each robot is bounded or not depends on the communication and the observation configurations of the entire multirobot system. To thoroughly study the covariance boundedness, we focus on a particular system with the widely-used motion and observation models. We then derive the covariance upper bound of the estimation covariance, and apply the result from the distributed estimation algorithm to obtain the boundedness criterion.

   % assumptions
   Specifically, we consider a system with unicycle motion model and the bearing-and-range measurements to demonstrate the analysis. We furthermore impose two assumptions:
   \begin{enumerate}
      \item Each robot has its orientation estimate, and the upper bound of the orientation estimate variance $\sigma_{\theta}^2$ is small and given.
      \item The observation and communication configurations are invariant over time, including the CI coefficients.
   \end{enumerate}
   As introduced in \cite{mourikis_performance_2006}, the first assumption decouples the position estimation from the orientation estimation, which is the main source of the linearization inconsistency problem \cite{julier_counter_2001, bailey_consistency_2006, huang_convergence_2007}. As the EKF heavily relies on the linearization approximation, the requirement of small orientation error also ensures the applicability of ongoing analysis. 
   The second assumption is imposed to assure that the entire system configuration is stationary. As a result, the boundedness analysis of the cooperative localization algorithm can be achieved by that of the distributed estimation algorithm \cite{chang_control_theoretical_2018}.
   
   % notation
   With the assumption that the orientation estimate is provided, all robots now estimate the same state, or the positions of all robots, denoted by $\xi_t= [p_{1,t}^{\T}, \dots, p_{N,t}^{\T}]^{\T}$. The estimate of $\xi_t$ at robot $i$ is $\hat{\xi}_{i,t}$, with mean $\bar{\xi}_{i,t}$ and covariance $\Phi_{i,t}$. While all the robots are estimating the same state, the communication step just degenerates to the vanilla CI step.

%%%
\subsection{System Model}   
   
   % motion propagation
   Given the velocity input $u_{i,t}$, the unicycle model describes the state propagation as
   \begin{equation}
      p_{i,t+1} 
      = \begin{bmatrix}
         x_{i,t} + (u_{i,t} + w_{i,t}) \Delta t \cos\theta_{i,t} \\
         y_{i,t} + (u_{i,t} + w_{i,t}) \Delta t \sin\theta_{i,t}
      \end{bmatrix},
      \label{eq:p_motion_model}
   \end{equation}
   where $w_{i,t}$ denotes the input noise and $\Delta t$ is the time interval between two consecutive update points.

   % observation model
   In terms of the observation model, we first set up a generic relative observation model, whose observability can be explicitly characterized. We then use the relative observation model as an intermediate step to analyze the bearing-and-range measurements.
   When robot $i$ observes object $j$, which can be either another robot or a landmark, the relative measurement $o_{ij,t}$ is given by
   \begin{equation}
      o_{ij,t} = C^{\T} (\theta_{i,t}) (p_{j,t} - p_{i,t} ) + v_{ij,t},
      \label{eq:p_relative_observation_model}
   \end{equation}
   where $C(\theta) = \begin{bmatrix} \cos \theta & -\sin \theta \\ \sin \theta & \cos \theta \end{bmatrix}$ is the rotation matrix.
   The observation noise $v_{ij,t}$ is a zero-mean Gaussian random vector with covariance $R_{v,ij}$.
   If robot $i$ observes object $j$ by the bearing measurement $\phi_{ij}$ and the range measurement $r_{ij}$, we characterize this measurement as
   \begin{align}
      o'_{ij,t} &= 
      \begin{bmatrix}
         \phi_{ij,t} \\
         r_{ij,t}
      \end{bmatrix} + v'_{ij,t} \notag \\
      &=
      \begin{bmatrix}
         \tan^{-1}\left(\frac{y_{j,t} - y_{i,t}}{x_{j,t} - x_{i,t}}  \right)- \theta_{i,t} \\
         \sqrt{ (x_{j,t} - x_{i,t})^2 + (y_{j,t} - y_{i,t})^2  }
      \end{bmatrix} + v'_{ij,t}.
      \label{eq:p_rnb_observation_model}
   \end{align}
   With the bearing measurement $\phi_{ij,t}$ and the range measurement $r_{ij,t}$, the relative measurement can be obtained by
   \begin{equation}
      o_{ij,t} = r_{ij,t}
      \begin{bmatrix}
         \cos( \phi_{ij,t}) \\
         \sin( \phi_{ij,t})
      \end{bmatrix},
      \label{eq:rnb_2_relative_o}
   \end{equation}
   together with the noise by linearizing (\ref{eq:rnb_2_relative_o})
   \begin{equation}
      v_{ij,t} = \begin{bmatrix}
          -r_{ij,t} \sin(\phi_{ij,t}) & \cos(\phi_{ij,t}) \\
           r_{ij,t} \cos(\phi_{ij,t}) & \sin(\phi_{ij,t})
      \end{bmatrix}
      v'_{ij,t}.
      \label{eq:rnb_2_relative_v}
   \end{equation}

%%%   
\subsection{Cooperative Localization Algorithm}

   % motion propagation
   We now apply the proposed cooperative localization algorithm on this particular model.
   By linearizing (\ref{eq:p_motion_model}), the error propagation equation of robot $j$ becomes
   \begin{equation}
      \tilde{p}_{j,t+1} 
      = \tilde{p}_{j,t} + \Delta t \begin{bmatrix}
         \cos\theta_{j,t} & -u_{j,t}\sin\theta_{j,t} \\
         \sin\theta_{j,t} & u_{j,t}\cos\theta_{j,t} 
      \end{bmatrix}
      \begin{bmatrix}
         \tilde{u}_{j,t} \\
         \tilde{\theta}_{j,t}
      \end{bmatrix}.
   \end{equation}
   For $j=i$, the odometry input $u_{i,t}$ is known, $\tilde{u}_{i,t} = w_{i,t}$, and therefore the covariance increment is given by
   \begin{equation}
      G_i Q_w G_i^{\T} = (\Delta t)^2 C(\theta_{i,t})
         \begin{bmatrix}
            \Sigma_w & 0\\
            0 & u_{i,t}^2\sigma^2_{\theta}
         \end{bmatrix}
         C^{\T}(\theta_{i,t}).
   \end{equation}
   For $j\neq i$, we model the odometry input $u_{j,t}$ itself as a random variable with variance $\sigma_u^2$, since it is not available for robot $i$. The covariance increment upper bound can then be taken as 
   \begin{equation}
      G_j Q_u G_j^{\T} = (\Delta t)^2 \max(\sigma_u^2, u_{\max}^2\sigma^2_{\theta}) I_2.
   \end{equation}
   In summary, the covariance is update by
   \begin{equation}
      \Phi_{i,t+1} = \Phi_{i,t} + \Diag(G_1 Q_u G_1^{\T}, \dots, G_i Q_w G_i^{\T}, \dots, G_N Q_u G_N^{\T}).
      \label{eq:phi_time_update}
   \end{equation}

   % observation
   As for the observation update, based on (\ref{eq:p_relative_observation_model}), the observation error can be linearly approximated as
   \begin{equation}
      \tilde{o}_{ij,t} \approx C^{\T}(\hat{\theta}_{i,t}) \check{H}_{ij} \tilde{\xi}_{i,t} + C^{\T}(\hat{\theta}_{i,t}) J \check{H}_{ij} \hat{\xi}_{i,t} \tilde{\theta}_{i,t} + v_{ij,t},
   \end{equation}
   to distinguish the estimation error $\tilde{\xi}_{i,t}$, the orientation estimation error $\tilde{\theta}_{i,t}$, and the measurement noise $v_{ij,t}$, where $J = \begin{bmatrix} 0 & 1 \\ -1 & 0 \end{bmatrix}$.
   If the observed object $j$ is a landmark,
   \begin{equation}
      \check{H}_{ij} = \begin{bmatrix} \\
            0_{2\times 2} & \cdots & \underbrace{-I_2}_i & \cdots & 0_{2\times2}
         \end{bmatrix}_{2\times 2N};
   \end{equation}
   while the observed object is robot $j$, then 
   \begin{equation}
      \check{H}_{ij} = \begin{bmatrix} \\
            0_{2\times 2} & \cdots & \underbrace{-I_2}_i & \cdots & \underbrace{I_2}_j & \cdots & 0_{2\times2}
         \end{bmatrix}_{2\times 2N}.
   \end{equation}
   
   The covariance of the innovation $\tilde{o}_{ij,t}$ is then given by
   \begin{equation}
      \E[\tilde{o}_{ij,t} \tilde{o}_{ij,t}^{\T}] = C^{\T}(\hat{\theta}_{i,t}) \check{H}_{ij} \Phi_{i,t} \check{H}_{ij}^{\T} C(\hat{\theta}_{i,t}) + R_{\theta_i,j,t} + R_{v,ij},
   \end{equation}
   where $R_{\theta_i,j,t} = \check{H}_{\theta_i, j, t} \sigma_{\theta_i,t}^2 \check{H}_{\theta_i, j, t}^{\T}$ and $\check{H}_{\theta_i, j, t} = C^{\T}(\hat{\theta}_{i,t}) J \check{H}_{ij} \hat{\xi}_{i,t} $.

   % multiple observations
   For multiple observation case, we stack the observation errors and obtain
   \begin{align*}
      \tilde{o}_{i,t} &\approx [C^{\T}(\hat{\theta}_{i,t}) \check{H}_{ij}]_{j \in O_{i}} \tilde{\xi}_{i,t} + [C^{\T}(\hat{\theta}_{i,t}) J \check{H}_{ij}]_{j \in O_{i}} \hat{\xi}_{i,t} \tilde{\theta}_{i,t} + v_{i,t} \notag\\
         &= \Xi_{i,t}^{\T} \check{H}_i \tilde{\xi}_{i,t} + \left( I_{|O_{i}|} \otimes C^{\T}(\hat{\theta}_{i,t}) J \right) \check{H}_i \hat{\xi}_{i,t} \tilde{\theta}_{i,t} + v_{i,t},
   \end{align*}
   where $v_{i,t} = [v_{ij,t}]_{j \in O_{i,t}}$,
   \begin{equation}
      \check{H}_i = \left[ \check{H}_{ij} \right]_{j \in O_i},
   \end{equation}
   $\Xi_{i,t} = I_{|O_{i}|} \otimes C(\hat{\theta}_{i,t})$, and $\otimes$ stands for the Kronecker product.
   The overall observation covariance can be expressed as
   \begin{equation}
      \E[\tilde{o}_{i,t} \tilde{o}_{i,t}^{\T}] = \Xi_{i,t}^{\T} \check{H}_i \Phi_{i,t} \check{H}_i^{\T} \Xi_{i,t}  + R_{\theta_i,t} + R_{v,i},
   \end{equation}
   where the first term comes from the position estimation error.
   The covariance is then updated in the observation update by
   \begin{equation}
      \Phi_{i,t^+}^{-1} = \Phi_{i,t}^{-1} + \check{H}_i^{\T} \Xi_{i,t} \left( R_{\theta_i,t} + R_{v,i} \right)^{-1} \Xi_{i,t}^{\T} \check{H}_i.
      \label{eq:phi_observation_update}
   \end{equation}   
   While the above derivations follow the relative observation model, the corresponding error approximation for bearing-and-range observation model can be obtained with (\ref{eq:rnb_2_relative_o}) and (\ref{eq:rnb_2_relative_v}).

   % communication update

%%%
\subsection{Covariance Upper Bound}

   As the matrix propagation of $\Phi_{i,t}$ involves time-dependent coefficients, we set up an upper bound matrix $\Psi_{i,t}$ of $\Phi_{i,t}$ with invariant coefficients. In the time propagation update, we can choose
   \begin{equation}
      \check{Q} = (\Delta t)^2 \max(\sigma_u^2, u_{\max}^2 \sigma^2_{\theta}) I_{2N}
   \end{equation}
   and update $\Psi_t$ by
   \begin{equation}
      \Psi_{i,t+1} = \Psi_{i,t} + \check{Q}
      \label{eq:psi_time_update}
   \end{equation}
   when $\Phi_{i,t}$ is updated by (\ref{eq:phi_time_update}).
   Similarly, for the observation update, while $\Phi_t$ is updated by (\ref{eq:phi_observation_update}), we can find a positive definite matrix $\check{R}_i$ such that $\check{R}_i^{-1} \leq \Xi_{i,t} \left( R_{\theta_i,t} + R_{v,i} \right)^{-1} \Xi_{i,t}^{\T}$, and update$\Psi_t$ according to
   \begin{equation}
      \Psi_{i,t^+}^{-1} = \Psi_{i,t}^{-1} + \check{H}_i^{\T} \check{R}_i^{-1} \check{H}_i.
      \label{eq:psi_obsv_update}
   \end{equation}
   For the communication step, $\Phi_{i,t}$ is updated by the conventional CI formula, and so it $\Psi_{i,t}$.

   By this construction, with the same initial condition, or $\Phi_{i,0} = \Psi_{i,0}$, we have $\Phi_{i,t} \leq \Psi_{i,t}$ for all $t$. In other words, $\Psi_{i,t}$ is an upper bound of $\Phi_{i,t}$. We then show the boundedness criterion of $\Psi_{i,t}$, which leads to the boundedness of $\Phi_{i,t}$.

%%%
\subsection{Covariance Boundedness Analysis}

   We now apply the result of the distributed Kalman filter with CI in \cite{chang_control_theoretical_2018} to analyze the covariance boundedness of $\Psi_{i,t}$. To explicitly characterize the relations among all robots, we use graphs to describe the observation and the communication configurations in the multirobot system. We define the observation graph and the communication graph separately to distinguish the observation and the communications relations. 
   We define the observation graph of robot $i$ as a graph $\mathcal{G}_{O_i}=\{\Omega^*, E_{O_i} \}$. The nodes of the graph $\Omega^*=\{1,\dots,n,\lambda\}$, which includes all the robots as well as the landmark. The pair $(i,j) \in E_{O_i}$ if $j \in O_i$. In other words, the links in the observation graph $\mathcal{G}_{O_i}$ stand for the observation from robot $i$ to entity $j$.\footnote{The definition of the observation graph is different from the observation topology defined in \cite{chang_control_theoretical_2018}, since the raw observations are not exchanged in this cooperative localization algorithm.} We also define the communication graph as a graph $\mathcal{G}_{C}=\{\Omega^*, E_{C} \}$, where $(i,j) \in E_{C}$ if $j \in C_i$. We then can use the following notation to collect all the robots that contribute the information to robot $i$ by the communication links.
   
\begin{definition}[Super Neighborhood \cite{chang_control_theoretical_2018}]
   For $j \neq i$, $j \in S_i$ if there exists a path in $\mathcal{G}_{C}$ from $j$ to $i$.
\end{definition}
   We define that $S^*_i = S_i \cup \{i\}$.

\begin{proposition}[The Boundedness Criterion]
   If the graph $\mathcal{G}_i = (\Omega^*, \cup_{j \in S^*_i} E_{O_j})$ is weakly connected,
   then $\Phi_{i,t}$ is bounded.
   \label{proposition}
\end{proposition}
\begin{proof}
   % 1st
   First of all, the equations (\ref{eq:psi_time_update}) and (\ref{eq:psi_obsv_update}) are exactly those equations (8) and (12) in \cite{chang_control_theoretical_2018}.
   Therefore, we can consider the upper bound of the cooperative localization $\Psi_t$ as a realization of the distributed estimation algorithm described in \cite{chang_control_theoretical_2018}, whose boundedness criterion has been established.
   
   % 2nd 
   We then show that $(F, [\check{H}_j]_{j \in S^*_i})$ is observable if the graph $\mathcal{G}_i$ is weakly connected. If the landmark $\lambda$ is not connected in $\mathcal{G}_i$, or no landmark is observed by any robot in $S^*_i$, then
   \[
      [\check{H}_j]_{j \in S^*_i} = D(\mathcal{G}_i) \otimes I_2
   \]
   up to row reordering, where $D(\mathcal{G}_i)$ is the incidence matrix of $\mathcal{G}_i$ \cite[p.~202]{Meyer_matrix_2000}.
   If the landmark $\lambda$ is connected in $\mathcal{G}_i$, 
   \begin{equation}
      [\check{H}_j]_{j \in S^*_i} = D(\mathcal{G}_i)' \otimes I_2
      \label{eq:incidence_matrix_equation}
   \end{equation}
   up to row reordering, where $D(\mathcal{G}_i)'$ is defined by removing the landmark row in $D(\mathcal{G}_i)$. Therefore, (\ref{eq:incidence_matrix_equation}) holds for all cases. If $\mathcal{G}_i$ is weakly connected, $D(\mathcal{G}_i)'$ is full rank, and $D(\mathcal{G}_i)' \otimes I_2$ is also full rank, which implies the observability of $(F, [\check{H}_j]_{j \in S^*_i})$.

   % 3rd
   Since $(F, [\check{H}_j]_{j \in S^*_i})$ is observable and $(F, Q^{1/2})$ is controllable, $\Psi_{i,t}$ is bounded by Theorem 1 in \cite{chang_control_theoretical_2018}, which implies that $\Phi_{i,t}$ is also bounded.
\end{proof}

% explanation of this theorem
   Proposition \ref{proposition} is given in a different form in our prior work \cite{tsangkai_1}, but a clearer treatment with graph theory is provided here. Proposition \ref{proposition} states that as long as all the information collected by robot $i$ covers the entire robot team, the information is sufficient enough to localize the entire robot team, which leads to bounded $\Psi_{i,t}$ and bounded $\Phi_{i,t}$ as well. Proposition \ref{proposition} also signifies that the information can come either from observation or from communication, and both sources contribute to the localization performance.

\section{Results}
\label{sec:experiment}

   In this section, we present the performance and the resilience of our algorithm as compared to other four state-of-the-art multirobot cooperative localization methods. Based on the state tracked in a single robot and the underlying method, for simplicity, we rename all $5$ algorithms as
   \begin{itemize}
      \item the local-state centralized equivalent (LS-Cen) \cite{roumeliotis_distributed_2002},
      \item the local-state covariance intersection (LS-CI) \cite{carrillo_arce_decentralized_2013},
      \item the local-state split covariance intersection (LS-SCI) \cite{li_split_2013},
      \item the local-state block diagonal approximation (LS-BDA) \cite{luft_recursive_2018},
      \item our global-state covariance intersection (GS-CI).
   \end{itemize}
   As the LS-Cen algorithm uses the entire available information without any approximation, the result of LS-Cen can be regarded as the optimal performance. We first simulate all methods with generated data, which not only shows that our algorithm requires far sparser communication topology to achieve comparable performance of other methods, but also visualizes the boundedness analysis in Sec. \ref{sec:analysis}. Next, we analyze all methods in a common multirobot dataset, and show that our algorithm is more resilient during unfavorable and adverse communication loss than other algorithms.

%%%
\subsection{Simulation}

\begin{figure}[t]
	\center{\includegraphics[width=0.4\textwidth]{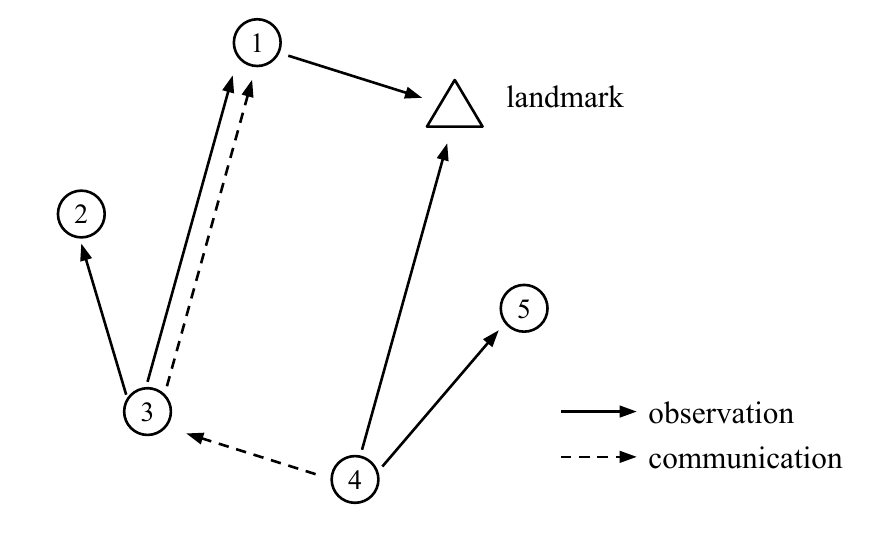}}
	\caption{The system topology with $N=5$ robots in the simulation. The communication graph is specified for the GS algorithm. For the LS algorithms, a fully connected communication graph is inherently required.}
	\label{sim:topology}
\end{figure}

\begin{figure}[t]
    \centering
    \includegraphics[width=0.45\textwidth]{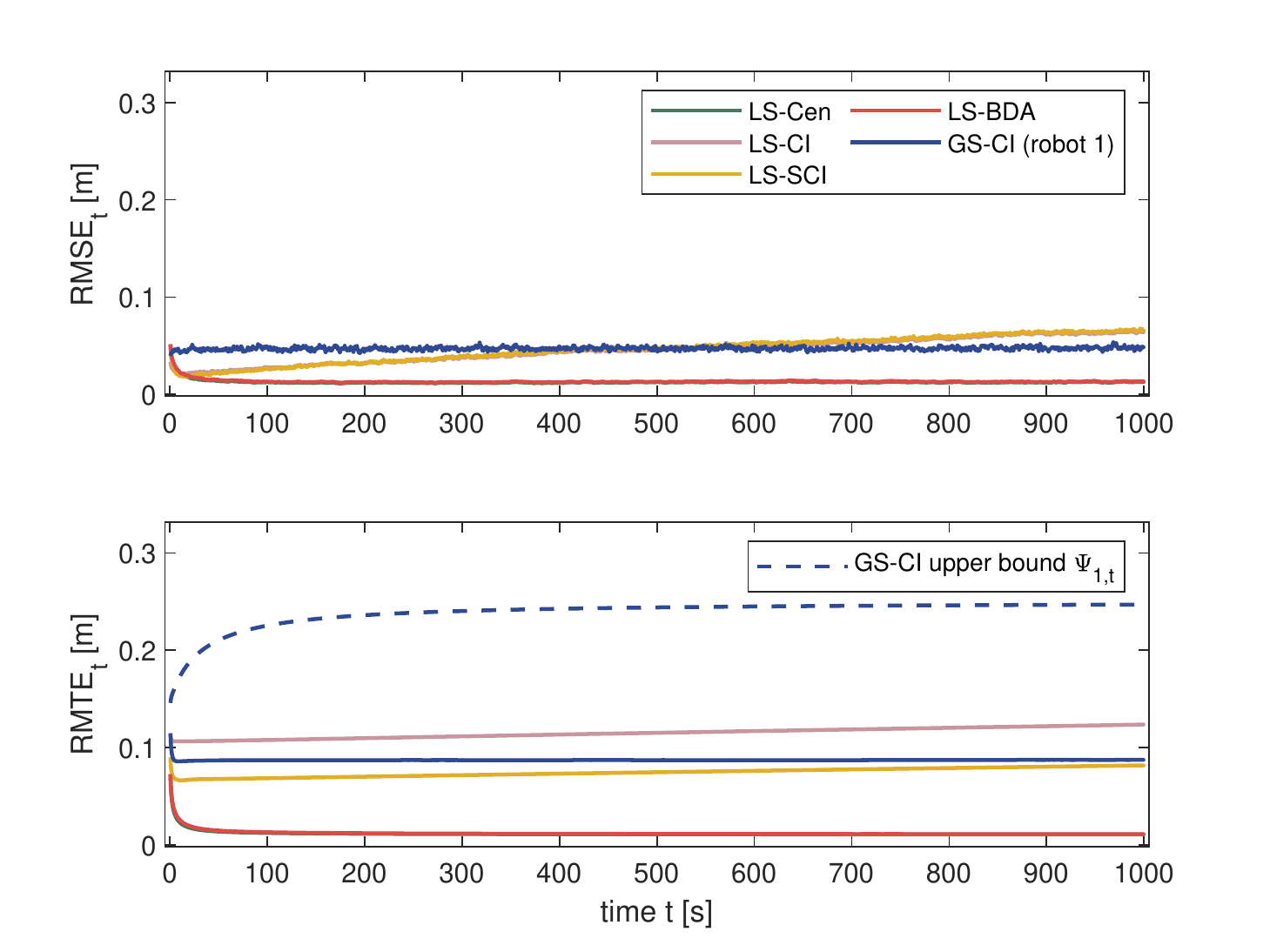}
	\caption{The cooperative localization performance with generated data. As for the communication graph, local-state (LS) algorithms assume all-to-all and perfect communication, and the global-state (GS) algorithm follows the graph in Fig. \ref{sim:topology}. The RMSE plot of LS-Cen and that of LS-BDA are overlapped. For the proposed GS-CI, robot $1$ has bounded covariance matrix, as suggested by Proposition 1.}
	\label{sim:boundedness}
\end{figure}

\begin{figure}[t]
    \centering
    \includegraphics[width=0.45\textwidth]{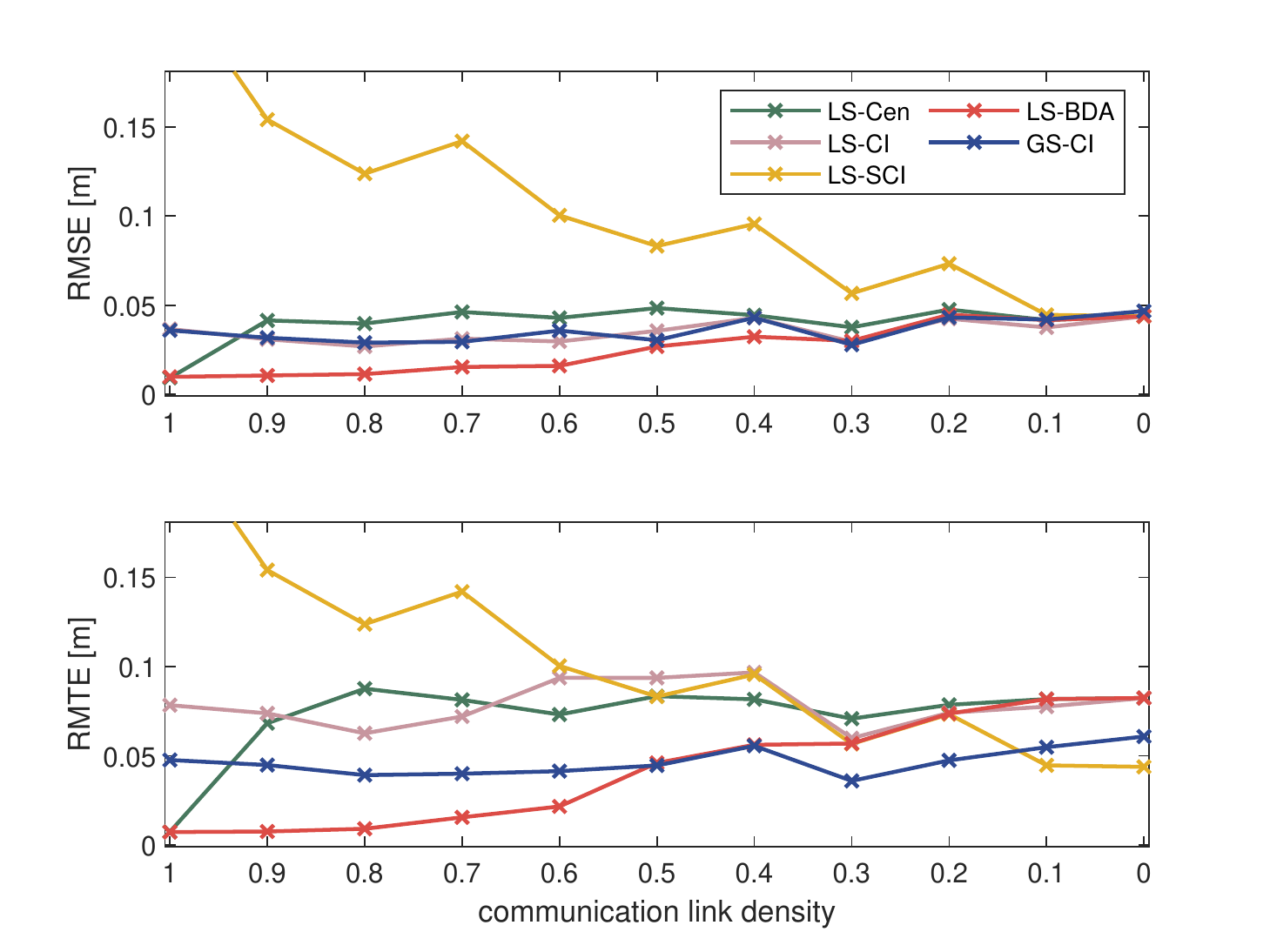}
	\caption{The averaged RMSE and RMTE at $t=1000$ sec over $100$ communication graphs. We fix the observation link density to be $0.75$. The lower the communication link density is, the sparser the communication graph becomes. The proposed GS-CI can provide good localization performance while ensuring the estimation consistency, even when the underlying communication graph is sparse.}
	\label{sim:topo}
\end{figure}

   To begin with, we investigate the performances of all $5$ algorithms with generated data.\footnote{The code of this subsection is available at https://github.com/tsangkai/multirobot\_localization.} In this simulation, we consider that the orientation estimate is given for $N=5$ robots, as assumed in Sec.~\ref{sec:analysis}. For each robot, the velocity input $u_{i,t}$ is taken uniformly between $[-0.09, 0.09]$ m/s, in which the velocity input variance in GS-CI can then be calculated. Fig. \ref{sim:topology} specifies the observation graph for the multirobot system. In terms of the communication graph, for local state (LS) algorithms, a fully connected communication graph is inherently required and therefore communication after each relative observation step is assumed to be perfect. For the global state (GS) algorithm, the communication is constrained as in Fig. \ref{sim:topology}, which is far sparser than those communication graphs for LS algorithms.

   To quantify the estimation performance against the ground truth, we define the root-mean-squared-error (RMSE) of the entire $N$ robots as
   \[
	  RMSE_t = \sqrt{\frac{{\sum_{i=1}^{N} \left\Vert [\bar{p}_{i,t}]_i - p_{i,t} \right\Vert}_2}{N}} ,
   \]
   where $[\bar{p}_{i,t}]_j$ is the estimate of $p_{i,t}$ by robot $j$.
   We also consider the root-mean-trace-error (RMTE) to capture the uncertainty evaluated in the algorithm, defined as:
   \[
	  RMTE_{t} = \sqrt{\frac{\sum_{i=1}^N \tr([\Phi_{i,t}]_i)}{N}},
   \]
   where $[\Phi_{i,t}]_j$ denotes the sub-covariance matrix of robot $j$ that relates to the position estimate of robot $i$ at time $t$. We plot the result in Fig. \ref{sim:boundedness}. In particular, for the GS-CI, we plot both the RMSE and the RMTE of robot 1 to discuss the boundedness analysis in Sec.~\ref{sec:analysis}.

\begin{table}
   \caption{Time-Averaged RMSE of UTIAS Datasets [m]}
   \label{table:rmse}
   \centering
    \begin{tabular}{ c  c c c c c } 
  sub-dataset  & LS-Cen & LS-CI & LS-SCI & LS-BDA & GS-CI \\ 
 \hline\hline
  1 & 1.28 & 1.67 & 1.12 & 1.31 &  1.42 \\ 
  2 & 0.74 & 1.41 & 1.75 & 0.80 &  0.79 \\
  3 & 0.23 & 0.96 & 1.23 & 0.26 &  0.29 \\
  4 & 0.21 & 1.21 & 1.49 & 0.23 &  0.28 \\
  5 & 1.72 & 5.45 & 5.20 & 1.79 &  2.17 \\
  6 & 0.79 & 2.08 & 2.07 & 0.82 &  0.85 \\
  7 & 0.59 & 1.49 & 1.73 & 0.86 &  0.82 \\
  8 & 0.71 & 0.96 & 2.00 & 0.84 &  0.80 \\
  9 & 0.26 & 0.28 & 0.65 & 0.27 &  0.31
   \end{tabular}
\end{table}

   % discussion: performance
   Based on the RMSE in Fig. \ref{sim:boundedness}, the LS-BDA and the proposed GS-CI show desirable results since their RMSEs remain relatively constant. However, the LS-BDA does not guarantee the estimation consistency, and achieve this performance with the fully-connected communication graph. Other CI-based methods, including LS-CI and LS-SCI, have increasing localization error over time, due to the overly conservative estimation as discussed in Sec. \ref{sec:related_work}.
   
   % discussion:
   Even though the proposed GS-CI shows desirable result, the required communication graph specified in Fig.~\ref{sim:topology} is far sparser in the GS-CI than those of the LS algorithms. Especially, as the graph $\mathcal{G}_1$ is weakly connected, Proposition 1 assures that the upper bound $\Psi_{1,t}$ is bounded, which leads to the boundedness of $\Phi_{1,t}$. In fact, besides the observation of the landmark, the rest of the information of robot $1$ comes from the single communication from robot $3$. This simulation thus shows how the observation and the communication are treated as complementary information sources in the proposed algorithm. In addition to the sparseness of the communication graph, the proposed GS-CI has the estimates of the entire robot team by design, which facilitates the cooperative planning within the multirobot system.

\begin{figure}[t]
    \centering
    \includegraphics[width=0.45\textwidth]{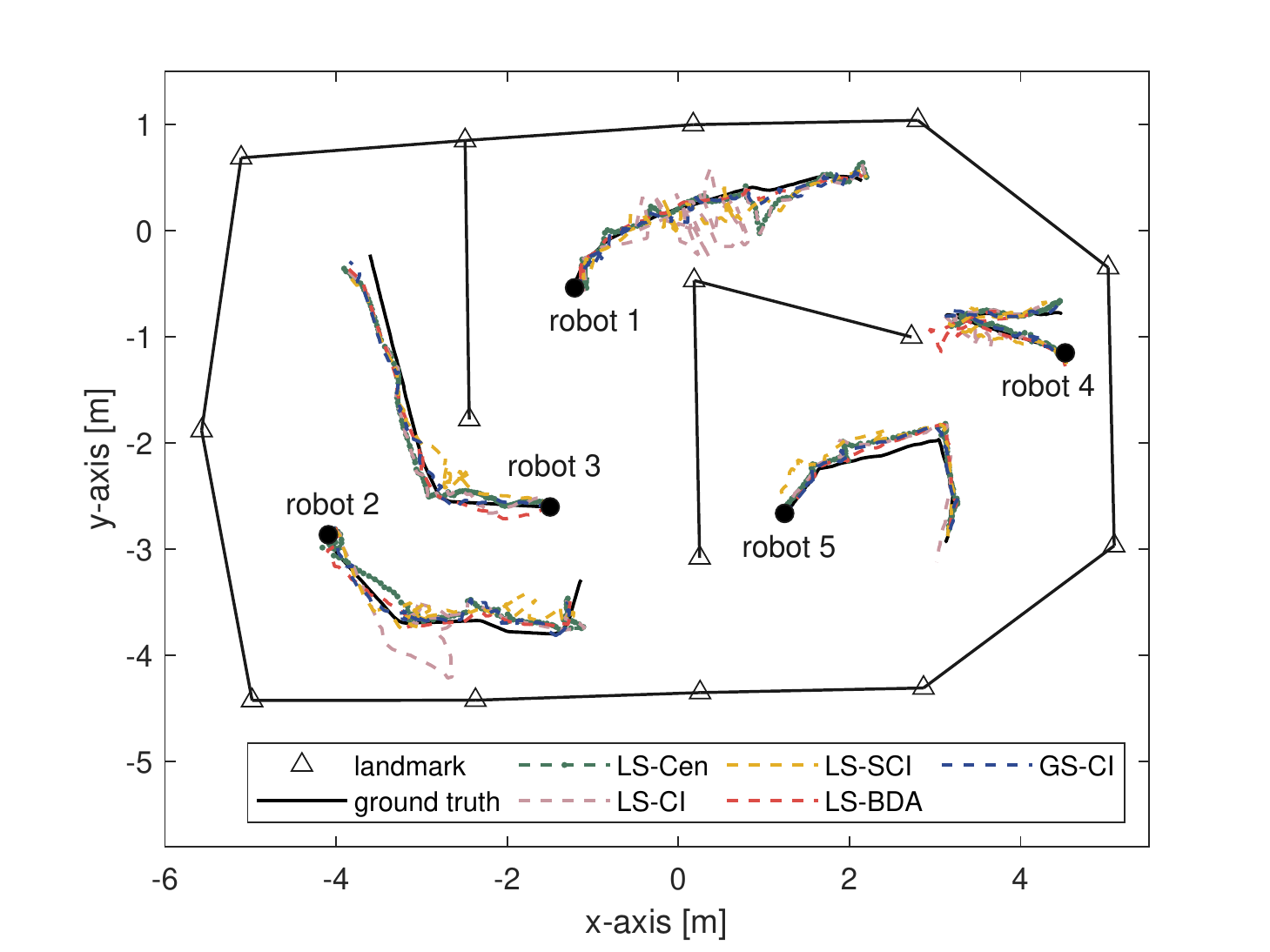}
	\caption{The trajectories of all $5$ robots with different localization algorithms in sub-dataset 9 between $1500$ and $1525$ sec. The communication is assumed to be available whenever needed. The proposed GS-CI is comparable to the LS-Cen, whose result is regarded as the best achievable performance.}
	\label{utias:trajectory_all}
\end{figure}      

\begin{figure}[t]
    \centering
    \includegraphics[width=0.45\textwidth]{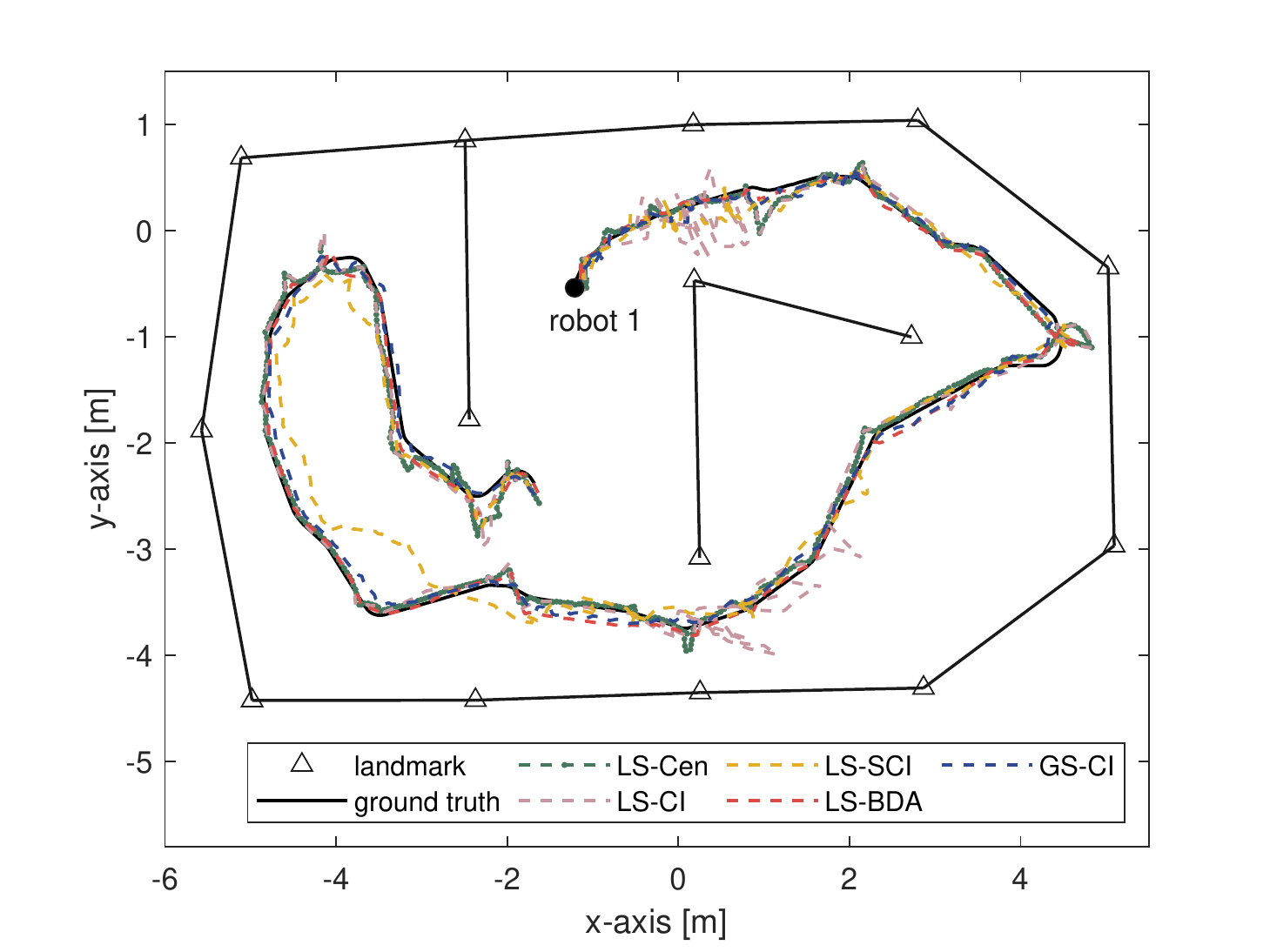}
	\caption{The trajectories of robot 1 with different localization algorithms in sub-dataset 9 between $1500$ and $1650$ sec. The communication is assumed to be available whenever needed. The proposed GS-CI is comparable to the LS-Cen, whose result is regarded as the best achievable performance.}
	\label{utias:trajectory_one}
\end{figure}

   % motivation
   In the previous simulation, the underlying communication graphs for LS algorithm are different from that of the GS algorithm. Since communication is required in the observation update for LS algorithms, the communication graph also affects the observation update. We now investigate the communication link requirement for all $5$ algorithms.
   In this setting, we randomly generate an observation and a communication graphs, and simulate all $5$ algorithms on the generated graphs. Each graph is generated by assigning a directional link between two nodes with a constant probability, or the link density. The observation updates of the LS algorithms are successful only if the underlying communication graph exists. While the landmark observation update depends on the exact system implementation, we assume that it is unaffected by the communication graph, and mainly focus on the relative observation update. In particular, the LS-Cen requires all-to-all communications after the observation update, and the LS-BDA needs a bidirectional communication link between the observation pair. As for the LS-CI and LS-SCI, only uni-directional communication link is sufficient to complete the relative observation.

   % result and discussion
   We simulate all $5$ algorithms with various communication link densities, and plot the averaged localization performance at $t=1000$ in Fig. \ref{sim:topo} over $100$ graphs. Since the LS-Cen requires the all-to-all communication graph, the estimation error of LS-Cen only significantly drop when the communication link density exceeds $0.9$. In other words, the success of the LS-Cen depends on a very dense communication graph. The problematic fusion scheme of LS-SCI becomes obvious when the communication graph becomes dense. Meanwhile, the localization performance of the LS-CI stays satisfying as the estimation consistency is maintained. The LS-BDA has the least estimation error in the simulation. The approximation of the LS-BDA to the LS-Cen becomes more accurate when the communication graph is dense.
    Overall, the proposed GS-CI can provide good localization performance while ensuring the estimation consistency, even when the underlying communication graph is sparse.

%%%
\subsection{Communication Resilience Experiment on the UTIAS Dataset}

   % dataset selction
   To demonstrate the resilience to communication failures of our algorithm, we use the University of Toronto Institute for Aerospace Studies (UTIAS) Multi-Robot Cooperative Localization and Mapping dataset \cite{leung2011utias}. This dataset is a cohesive collection of odometry and observation data from $N=5$ robots, together with accurate ground truth data of robot and landmark positions. This dataset is also widely used across several works as a common benchmark dataset.

   % table
   We first test those $5$ algorithms on the entire $9$ sub-datasets with all communication available on the first $500$ sec.\footnote{The code of this subsection is available at https://git.uclalemur.com/kjchen/tro2020/tree/master/v3.} Each algorithm estimates both the orientation and the position, and we mainly consider the position estimation here. We record the time-averaged $RMSE_t$ in Table \ref{table:rmse} for all $9$ sub-datasets. As expected, the LS-Cen algorithm has the lowest localization error in the entire $9$ sub-datasets. Overall, the proposed GS-CI has comparable localization performance compared to the LS-Cen, which is consistent with the previous simulation.

\begin{figure}[t]
    \centering
    \includegraphics[width=0.45\textwidth]{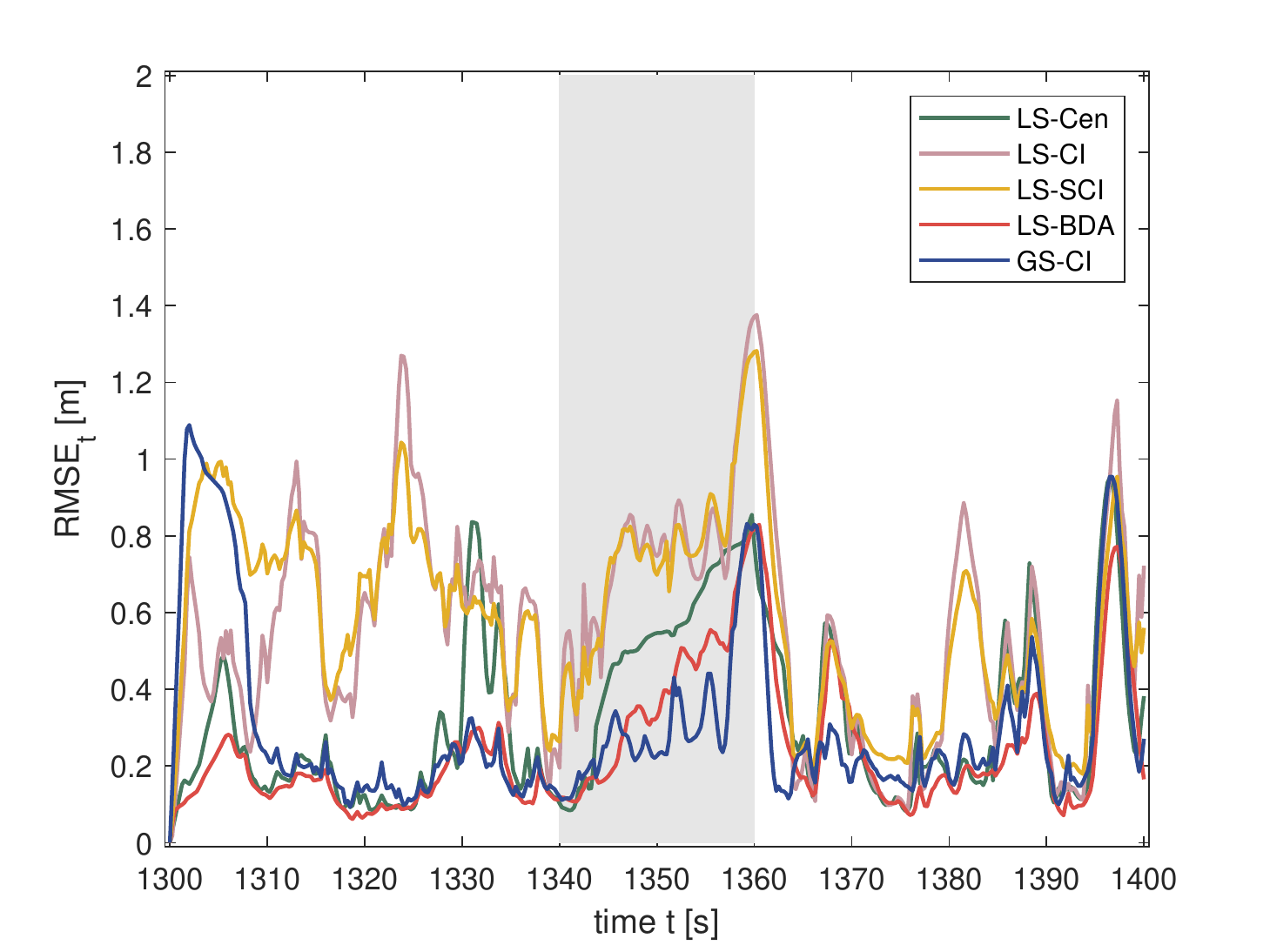}
	\caption{The RMSE with blocked communication from $1340$ to $1360$ sec of sub-dataset 9. The communication remains available besides the window between $1340$ to $1360$ sec. The proposed GS-CI shows resilience during this $20$ sec time window by separating the communication update and the observation update.}
	\label{utias:comm_loss}
\end{figure}

\begin{figure}[t]
    \centering
    \subfloat[\label{1a} $\rho=0.1$]{%
        \includegraphics[width=0.45\textwidth]{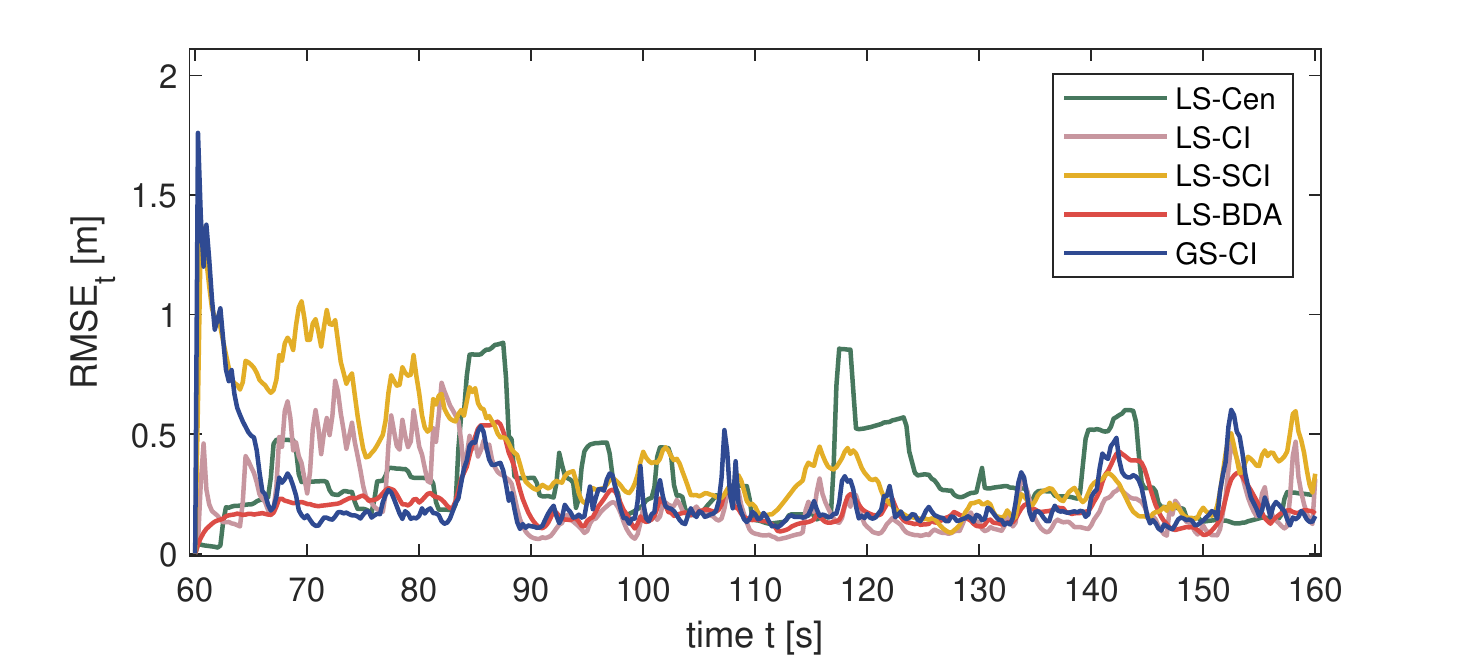}}
    \\
  \subfloat[\label{1c} $\rho=0.9$]{%
        \includegraphics[width=0.45\textwidth]{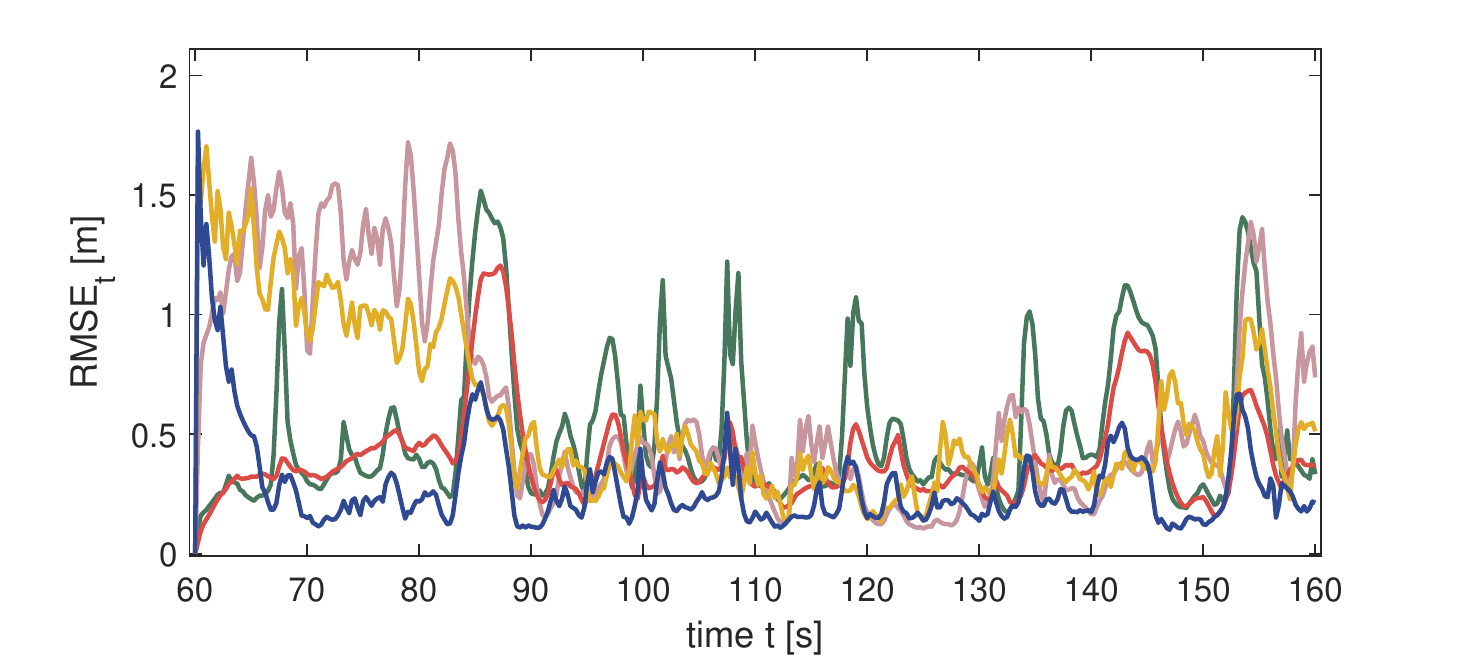}}
  \caption{The RMSE of a $100$ sec snapshot in sub-dataset 9 with two different communication link failure probabilities $\rho$. As there are more communication failures, the estimation error is larger with $\rho=0.9$ that that with $\rho=0.1$ for all algorithms. However, algorithms are affected differently. For instance, between $140$ and $150$ sec, the proposed GS-CI does not have a significant increase in the estimation error, and thus shows its resilience.}
  \label{utias:rmse_dynamics} 
\end{figure}

   % sub-dataset selction, Fig 3 and 4
   Among all $9$ sub-datasets in the UTIAS dataset, sub-dataset 9 is the only one that contains barriers, thus creating a more challenging scenario with its occasional occlusions in observations. We therefore select sub-dataset 9 to demonstrate the communication resilience in the following. To visualize this sub-dataset as well as the localization algorithms, we plot the estimated trajectories of all $5$ robots in Fig.~\ref{utias:trajectory_all} for a $25$ sec window. We also extend the time window of robot $1$ for an additional $125$ sec to show a longer trajectory in Fig.~\ref{utias:trajectory_one}. Both figures shows that the proposed GS-CI is comparable to the LS-Cen, whose result is regarded as the best achievable performance in the ideal scenario.

   % communication entirely blocked
   To investigate the communication resilience of each algorithm, we consider the scenario where the communication is blocked from an adverse source, and study the localization performance dynamics during this period. While different time windows show similar trends, we plot the time window between $1300$ and $1400$ sec of sub-dataset 9 in Fig.~\ref{utias:comm_loss} as an example. During the entire $100$ sec time window, the communication is entirely blocked from $1340$ to $1360$ sec, while the communication remains available for the rest of the time. In this $20$ sec window, which is marked as shaded area in Fig.~\ref{utias:comm_loss}, the estimation errors of all cooperative localization algorithms increase, but the proposed GS-CI has the lowest slope. In other words, by separating the communication update and the observation update, our algorithm is less susceptible from the communication unavailability but continues integrating information from the observation updates. For LS algorithms, since communication is essential to complete the some observation updates, the localization performances are largely impaired in this $20$ sec window.

\begin{figure}[t]
    \centering
    \includegraphics[width=0.45\textwidth]{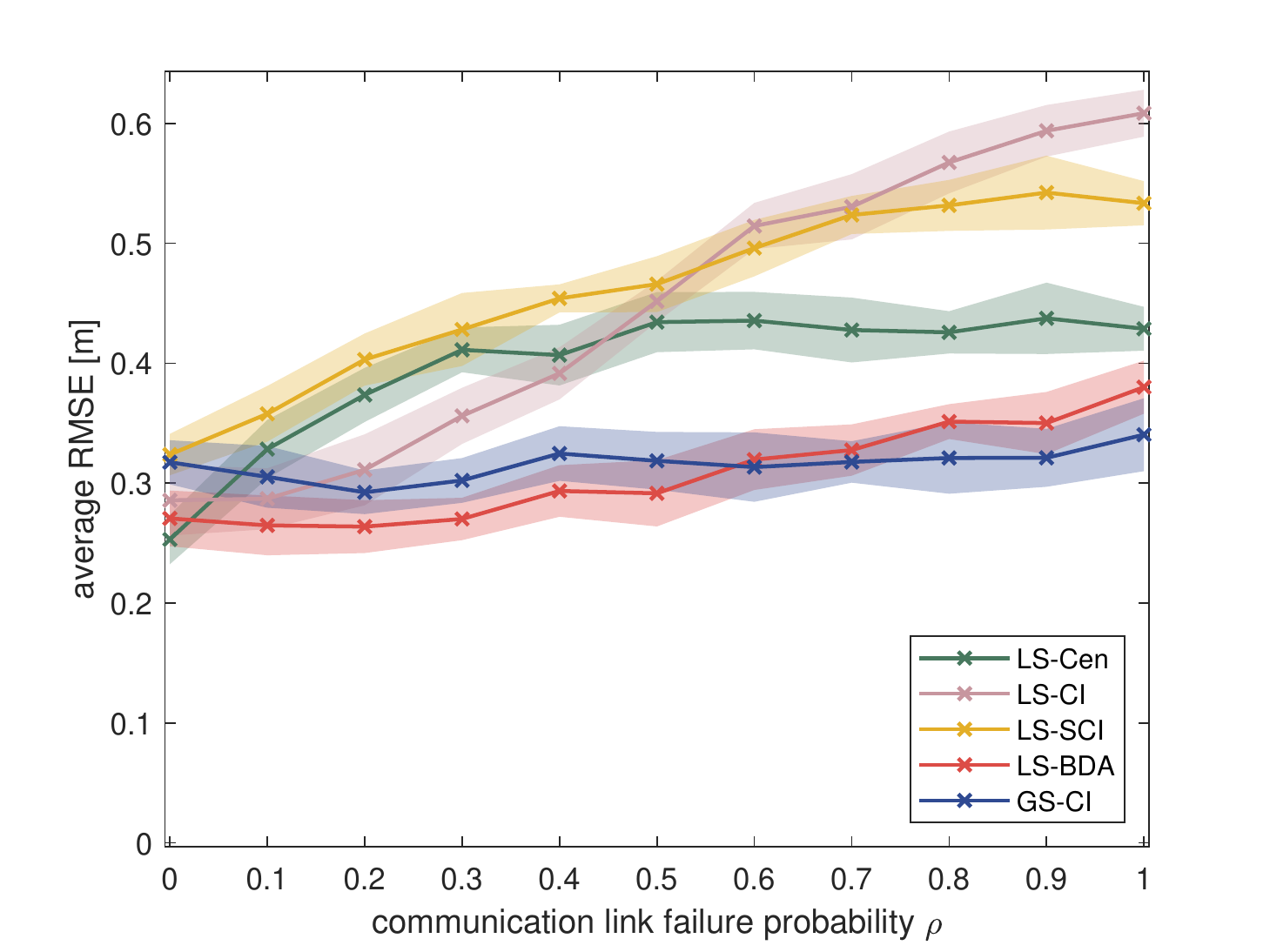}
	\caption{The time-averaged RMSE with varying communication link failure probabilities $\rho$ of the first $500$ sec of sub-dataset 9. We analyze every $50$ sec and plot the 3 standard deviation error bar for all $10$ windows. The proposed GS-CI is only slightly affected by the increase of $\rho$, and it shows resilience across different $\rho$ values, especially in unfavorable communication conditions.}
	\label{utias:rmse_rho}
\end{figure}

   % varying communication link failure probability
   We furthermore generalize the previous experiment and consider the effect of the communication link failure probability on those cooperative localization algorithms. In particular, we consider the scenario in which all the communication links between two robots exist, but suffer from failures with a constant probability $\rho$. For instance, the number of communications after the relative observation of LS-BDA is $2$. Therefore, with probability $(1-\rho)^2$, the relative observation update of LS-BDA can be completed successfully without communication failure.

   To emphasize the effects on estimation dynamics, we plot the $100$ sec snapshots with $\rho=0.1$ and $\rho=0.9$ in Fig.~\ref{utias:rmse_dynamics}. The former case with $\rho=0.1$ is close to the ideal case where all the communication is assumed perfect, while the later case with $\rho=0.9$ is similar to the $20$ sec window with blocked communication in Fig.~\ref{utias:comm_loss}. By comparing the two snapshots, the effect of the communication link failure probability $\rho$ on the cooperative localization algorithms becomes noticeable. For instance, between $140$ and $150$ sec, all the estimation errors increase with $\rho=0.9$ due to communication failures, but the resilience of each algorithm differs. Among all LS algorithms, the LS-BDA shows its estimation accuracy when $\rho=0.1$. However, while the LS-BDA has comparable performance to our GS-CI with $\rho=0.1$, it has overall worse localization performance with $\rho=0.9$. Such comparison substantiates the resilience of our GS-CI under the communication failure.

   To characterize the resilience performance under various scenarios, we plot the time-averaged RMSE against the communication link failure probability $\rho$ on the first $500$ sec of sub-dataset 9 in Fig. \ref{utias:rmse_rho}. In general, the increase of the communication link failure probability $\rho$ has negative impact on all algorithms, as the information coming from the communication becomes less available. However, as the communication failure probability $\rho$ increases, the LS-Cen and the LS-BDA algorithms suffer from higher localization error, even though they show superb localization performance in the ideal cases. On the contrary, the proposed GS-CI maintains a relatively flat curve as the communication failure probability $\rho$ increases. As the proposed GS-CI is only slightly affected by the increase of $\rho$, it shows resilience across different $\rho$ values, especially in unfavorable communication conditions.

\section{Conclusion}

   We present a multirobot cooperative localization algorithm that has an explicit communication update and preserves estimation consistency. By separating the communication and observation steps, the proposed algorithm naturally has better resilience to communication failures, which is inevitable in real-world scenarios. At the same time, the estimation consistency is guaranteed by the covariance intersection. We also characterize the boundedness criterion to demonstrate that communication and observation complementarily provide information in the proposed algorithm.

   The explicit communication in the proposed algorithm not only enhances its resilience to communication failures, but it also induces more design flexibility. For example, advanced scheduling of communication and observation becomes possible to further improve the localization performance as well as reduce overall operation cost. Our initial investigation is summarized in \cite{tsangkai_2}, and a thorough investigation will be completed in the future study. 
   
   % multi hop   
   In a multirobot system, the spatial states of other robots are often required for high-level goals, for example coverage control and cooperative path planning. For algorithms tracking only local states, additional communication has to be performed to acquire such information. As the proposed algorithm already tracks the state of the entire robot team, it actually provides a seamless integration for these tasks. Therefore, we are looking forward to applying our algorithm on cooperative multirobot systems to applications beyond localization.

   % As multirobot systems become sophisticated and mature, more cooperation within the robots is essential. However, the more cooperative a system is, the more vulnerable the system becomes. This work shows the way of cooperation while respecting the distributed nature of the multirobot system to ensure resilience to inter-robot communication failures. With this work targeting at the localization problem, we are interested in more advanced designs and implementations in multirobot systems that take both cooperation and resilience into consideration.

% \section*{Acknowledgment}
%   The authors would like to thank the reviewers for their constructive suggestions that greatly improve the quality of the manuscript.

\bibliographystyle{IEEEtran}
\footnotesize \bibliography{ref}

% Can be used to pull up biographies so that the bottom of the last one
% is flush with the other column.
% \enlargethispage{-5in}

\end{document}